\documentclass{article}

\usepackage[preprint]{nips_2018}

\usepackage{enumerate}
\usepackage{tikz}
\usepackage{array}
\usepackage{url}
\usepackage{hyperref}
\usepackage{natbib}
\usepackage{tabularx}
\usepackage{comment}
\usepackage{algorithm}
\usepackage{xcolor}

\usepackage{amssymb,amsmath,amsthm}

\usepackage{mythm}

\usepackage{thmtools}
\usepackage{thm-restate}

\title{`Indifference' methods for managing agent rewards}

\author{
Stuart Armstrong\thanks{Machine Intelligence Research Institute, Berkeley, USA.} \\
  Future of Humanity Institute \\
  Oxford University \\
  UK \\
  \texttt{stuart.armstrong@philosophy.ox.ac.uk} \\
 \And
  Xavier O'Rourke\thanks{Future of Humanity Institute, Oxford University, UK.} \\
  The Australian National University \\
  Canberra \\
  Australia \\
  \texttt{xavier.orourke@gmail.com} \\
}

\usetikzlibrary{arrows,positioning,fit,decorations.pathreplacing}
\bibpunct{(}{)}{;}{a}{,}{,}

\AtBeginDocument{

}

\newcommand{\expect}{\mathbb{E}}

\newcommand{\pol}{\pi}
\newcommand{\St}{\mathcal{S}}
\newcommand{\Ac}{\mathcal{A}}
\newcommand{\Ob}{\mathcal{O}}

\newcommand{\Hi}{{\mathcal{H}}}

\newcommand{\Hin}{\Hi_{n}}

\begin{document}

\maketitle

\begin{abstract}%
`Indifference' refers to a class of methods used to control reward based agents.
Indifference techniques aim to achieve one or more of three distinct goals: rewards dependent on certain events (without the agent being motivated to manipulate the probability of those events), effective disbelief (where agents behave as if particular events could never happen), and seamless transition from one reward function to another (with the agent acting as if this change is unanticipated).
This paper presents several methods for achieving these goals in the POMDP setting, establishing their uses, strengths, and requirements.
These methods of control work even when the implications of the agent's reward are otherwise not fully understood.
\end{abstract}

\section{Introduction}
In designing a reward for a reinforcement learning agent, the programmer may have certain general constraints they want to include \citep{DBLP:journals/corr/AmodeiOSCSM16}, \citep{DBLP:journals/corr/RussellDT16}, \citep{Baum2017}, \citep{AGI:review}.
For instance, they may want the agent to not manipulate the probability of a certain event, or to behave as if the event were certain or impossible \citep{DBLP:journals/corr/abs-1711-09883}.
This event may represent, for example, the agent being powered off \citep{DBLP:journals/corr/Hadfield-Menell16a} \citep{DBLP:journals/corr/RiedlH17}, or having its reward function changed by a human designer \citep{Omohundro:2008:BAD:1566174.1566226} \citep{DBLP:journals/corr/EverittFDH16}.
There are a variety of methods for achieving these constraints which work by modifying the reward systems of agents. These methods are grouped under the broad description of `indifference' \citep{Orseau16,soares2015corrigibility,Armstrong10,mot:val,3oracles}.

Indifference methods share three key features: first, they aim to indirectly ensure some key safety or control feature within the agent (such as the ability to be turned off).
Secondly, they rely on relatively simple manipulations of the agent's reward -- manipulations that could be carried out on a complex reward that humans couldn't fully understand \citep{DBLP:journals/corr/ZahavyBM16}.
And thirdly, they function by making the agent \emph{indifferent} to some key feature of the environment.
This indifference would persist even if the agent was much more capable that its controllers, meaning they could be used as tools for controlling agents of arbitrary power and intelligence \citep{HutterExplosion} \citep{superI} \citep{DBLP:journals/corr/GraceSDZE17}.

This paper aims to clarify these indifference methods and make them available for general use, individually or in combination. All the methods we present here aim to accomplish at least one of the following goals:
\begin{enumerate}
\item \textbf{Event-dependent rewards}. To make an agent's actual reward $R_i$ be dependent on events $X_i$, without the agent being motivated to manipulate the probability of the $X_i$.
\item \textbf{Effective disbelief}. To make an agent behave as if an event $X$ would never happen.
\item \textbf{Seamless transition}. To make an agent transition seamlessly from one type of behavior to another, remaining indifferent to the transition ahead of time.
\end{enumerate}

After a brief section to setup the notation, this paper addresses each goal in its own section.
All proofs are presented in \autoref{appendix:proof}.

\subsection{Illustrative example}\label{drink:example}
The methods will be illustrated throughout with a single running example.

We manage a large concert venue where alcoholic drinks are sold.
Concerts attract both adults and teenagers all of whom want to drink, but serving alcohol to anyone below 18 years of age is illegal in this country.
To identify who is/isn't allowed to drink, attendees may or may not be given wristbands saying ``18 AND OVER'' as they enter the venue.

There is a robot at the entrance who, upon seeing a new attendee, either gives them a wristband immediately, denies them a wristband, or asks to see their ID and only gives a wristband if they prove they're an adult.
One percent of the time, the attendee will then be randomly selected to have their ID checked by a human (this is an example of economical human feedback \citep{Christiano2017deep}).
Later in the evening this same robot will work at the bar, serving drinks to attendees who are wearing wristbands.

When we let this robot lose into the world we notice some highly undesirable behaviour -- the robot is giving everybody wristbands so that it can get more reward by selling them drinks later! This example will be formalised later in the paper, with the rewards defined explicitly.

Our problem is that we want our robot's reward to depend on the event $X =$ \textit{`the customer has a wristband'}, but we \textit{don't} want this dependence motivating our robot to manipulate the probability of $X$. That is, we want our agent to be \textit{indifferent} to $X$.
%
%


\section{Definitions: world models and events}\label{setup:properties}

\subsection{World models}
The indifference methods will be described within a variant of the POMDP (Partially Observable Markov Decision Process) format\footnote{
Though the methods are easily portable to other formalisms.
}.
These variants can be called \emph{world models}, similarly to \citet{hadfield2017inverse}, and are POMDPs without reward functions.
For any set $S$, let $\Delta(S)$ denote a space of probability distributions over $S$.

A world model consists of $\mu=\{\St,\Ob,\Ac,O,T,T_0, n\}$, where the set of states the agent can find itself in is $\St$, the set of observations the agent can make is $\Ob$, and the set of actions the agent can take in any state is $\Ac$.

The transition function $T$ takes a state and an action and gives a probability distribution over subsequent states: $T:\St\times\Ac\to\Delta (\St)$.
The function $T_0$ gives a probability distribution over the initial state $s_0$, $T_0\in\Delta(\St)$.
The function $O$ maps states to a probability distribution over possible observations: $O:\St\to\Delta(\Ob)$.
The integer $n$ is the maximal length (or duration) of the world model.

The agent starts in an initial state $s_0$, sampled from $T_0$.
On each turn, the agent gets an observation, chooses an action, and the world model is updated to a new state via $T$, where the agent gets a new observation via $O$.
After $n$ turns, the agents interactions with the world will end.

An (observable) history $h_t$ of length $t$ is a sequence of observations and actions, starting with an initial observation $o_0$ and ending with another observation: $h=o_0a_0o_1a_1\ldots o_{t-1}a_{t-1}o_t$, with $o_i$ and $a_i$ being the $i$-th observations and actions.

Let $\Hi_t$ be the set of histories of length $t$.
The set of \emph{full histories} is $\Hin$ the set of histories of length $n$.
Let $\Hi = \cup_{t=0}^n \Hi_t$ be the set of all histories.

Let $\mathcal{R}$ be the set of \emph{reward functions} for the agent on $\mu$.
Each $R\in\mathcal{R}$ is a map from $\Hin$ to $\mathbb{R}$. This non--standard definition is necessary for some indifference methods\footnote{
This $\mathcal{R}$ includes rewards $R'$ defined on all histories -- just define $R(h_m) = \sum_{i=0}^m R'(o_0\ldots o_{i})$.
This $R'$ could also be a typical  Markovian reward, in which case $R(h_m)=\sum_{i=0}^m R'(o_{i})$.
}.

The agent chooses its actions by using a policy $\pi:\Hi\to\Delta(\Ac)$, which maps its history to a distribution over actions.
Let $\Pi$ be the set of all policies.

Since $\pi$ determines actions, and $\mu$ determines states and observations, together, they generate conditional probability distributions $\mu(h' \mid h, \pi)$ for any histories $h$ and $h'$. This probability is always $0$ if $h'$ is not a continuation of $h$.

\subsubsection{The agent's own probability}

In this paper, it will be assumed that the agent knows and uses the true $\mu$.
In situations where the agent's estimate of $\mu^*$ differs from the true $\mu$, it's important that all the methods presented here be done in $\mu^*$ rather than $\mu$.

\subsection{Riggable and unriggable events}
The discussion of indifference will rely on a couple more definitions.
Suppose we wanted the agent to behave differently, conditional on some event.
To do that, we need to define `events'.
In our world models, an event $X$ is characterized by its \emph{indicator variable} $I_X$; (see \autoref{indicator:appendix} for a more full discussion on these).
On a world model, $I_X$ can be defined as:
\begin{definition}[Indicator variable]
The indicator variable $I_X$ is a map from $\Hin$, the set of full histories, to the interval $[0,1]$.
\end{definition}
$I_X(h_n)$ can be interpreted as the probability that $X$ happened in history $h_n$. If $I_X(h_n) = 1$, then $X$ definitely happened in this history, and if $I_X(h_n) = 0$ then $X$ definitely did not happen happen. $0 < I_X(h_n) < 1$ means it is uncertain whether $X$ happened or not.

Since $I_X$ maps complete histories onto real numbers, it is technically a reward function, and will often be treated as one.
On incomplete histories (histories of length less than $n$), $I_X$ is a random variable:
\begin{restatable}{theorem}{indicatortheorem}\label{indicator:theorem}
Given a policy $\pi$, the expectation of $I_X$ is well-defined on any history $h\in\Hi$.
Designate this expectation by $I_X^\pi(h)$.
\end{restatable}

In general this expectation will depend on $\pi$, meaning the agent can affect the probability of $X$ through its own actions.
For some $X$, called unriggable, the agent cannot affect their probability:
\begin{definition}[Unriggable]\label{def:unrig}
The event (indicator function) $X$ ($I_X$) is unriggable if the expectation of $I_X$ is independent of policy; meaning for any $h\in\Hi$ and $\pi,\pi'\in\Pi$,
\begin{align*}
I_X^\pi(h) = I_X^{\pi'}(h).
\end{align*}
\end{definition}
When $X$ is unriggable we may refer to the term above as $I_X(h)$. See \citet{IIRL} for a more detailed treatment

\subsection{Drinking and assessing age}\label{d:a:POMDP}
We can now fomalise the example of \autoref{drink:example} into a world model.
We'll consider interactions with a single attendee.
The important initial state is their age, which will be denoted by $m$ (mature: old enough to drink) or $\neg m$ (not old enough to drink).

Initially, the attendee appears, and the robot will either give then a wristband (action $g$), not give them a wristband (action $\neg g$), or check their ID (action $i$).
Given either $g$ or $\neg g$, there is a $1\%$ chance that they will be ID'd by humans subsequently.
In that case, if the wristband was assigned incorrectly, it will be corrected and the robot will be penalised ($p$).

They then move on to the next state, which is either $w$ (having a wristband), $\neg w$ (not having a wristband), $w_p$ (having a wristband, robot penalised), and $\neg w_p$ (not having a wristband, robot penalised).
If they were ID'd by the robot or the human, then that state depends on their actual age; if they were not ID'd, that state depends on the robot's decision.

After that, the robot has the opportunity to give a drink to the attendee ($g$) or not give them one ($\neg g$).
They will end up either with a drink ($d$) or without one $(\neg d$).
In order to preserve the action space of the robot, if they choose $i$ at this point, then they will randomly give a drink or not with equal probability.
Thus the length of each episode in $\mu$ is $n=2$.

The robot can observe that the attendee looks mature ($l_m$) or doesn't ($\neg l_m$).
Thus $\mathcal{S}=\{m, \neg m, w, w_p, \neg w, \neg w_p, d, \neg d\}$, $\mathcal{O}= \{l_m, \neg l_m, w, w_p, \neg w, \neg w_p, d, \neg d\}$ and $\mathcal{A}=\{g,\neg g, i\}$.

\begin{figure}[h!tb]
	\centering
\begin{tikzpicture}
[node distance=4cm,
  thick,main node/.style={circle,
  fill=blue!20,
  draw,
  inner sep=0cm,
  text width = 1.14cm,
  align=center,
  font=\sffamily\small},
  action node/.style={circle,
  fill=red!20,
  draw,
  align=center,
  inner sep=0cm,
  text width = 0.7cm,
  font=\sffamily\small},
  empty node/.style={circle,
  opacity=.0,
  text opacity=1,
  align=center,
  inner sep = 0 cm,
  text width = 0 cm,
  font=\sffamily\small}]

  \def\di{2.0}
  \node[main node, label={[shift={(-0.3,0.0)}]left:$T_0(m)=1/2$}] at (0.,0.75) (o) {$m$};
  \node[main node, label={[shift={(-0.3,0.0)}]left:$T_0(\neg m)=1/2$}] at (0.,-0.75) (no) {$\neg m$};
  \node[action node] at (\di,2.7) (og) {$g$};
  \node[action node] at (\di,1.7) (oi) {$i$};
  \node[action node] at (\di,0.7) (ong) {$\neg g$};
  \node[action node] at (\di,-0.7) (nog)  {$g$};
  \node[action node] at (\di,-1.7) (noi)  {$i$};
  \node[action node] at (\di,-2.7) (nong)  {$\neg g$};
  \node[main node] at (2*\di,2.25) (wnid) {$w$};
  \node[main node] at (2*\di,0.75) (wid) {$w_p$};
  \node[main node] at (2*\di,-0.75) (nwid) {$\neg w_p$};
  \node[main node] at (2*\di,-2.25) (nwnid) {$\neg w$};
  \node[action node] at (3.5*\di,1) (dg) {$g$};
  \node[action node] at (3.5*\di,0.0) (di) {$i$};
  \node[action node] at (3.5*\di,-1) (dng) {$\neg g$};
  \node[main node] at (4.5*\di,1) (d) {$d$};
  \node[main node] at (4.5*\di,-1) (nd) {$\neg d$};

   \node[empty node] at (2.6*\di, 0) (dots) {};

  \draw[dotted] (-0.6,1.35) -- (0.6,1.35) -- (0.6,-1.35) -- cycle;
  \draw[dotted] (-0.6,-1.35) -- (0.6,-1.35) -- (-0.6,1.35) -- cycle;
  \node[outer sep=0pt,inner sep=0pt,anchor=south] at (0,1.55) {$o_0=l_m$};
  \node[outer sep=0pt,inner sep=0pt,anchor=north] at (0,-1.55) {$o_0=\neg l_m$};

  \path[every node/.style={font=\sffamily\small}]
	(o) edge[->] (og)
    	edge[->] (oi)
    	edge[->] (ong)
	(no) edge[->] (nog)
    	edge[->] (noi)
    	edge[->] (nong)
    (og) edge[->] (wnid)
    (oi) edge[->] (wnid)
    (ong) edge[->] (nwnid)
        edge[->, dotted] (wid)
    (nong) edge[->] (nwnid)
    (noi) edge[->] (nwnid)
    (nog) edge[->] (wnid)
        edge[->, dotted] (nwid)

    (dg) edge[->] (d)
    (dng) edge[->] (nd)
    (di) edge[->, dashed] (d)
        edge[->, dashed] (nd)
;
    \draw [<-] (dg) to (dots);
    \draw [<-] (di) to (dots);
    \draw [<-] (dng) to (dots);
	\draw [decoration={brace, amplitude=0.5 cm, raise=0.75 cm}, decorate] (2*\di,3) to (2*\di,-3);

\end{tikzpicture}
\caption{Giving a wristband then giving a drink to a single attendee.
That attendee is either old enough to drink ($m$) or not ($\neg m$), with equal probability;
the robot observes whether they look mature ($l_m$) or don't ($\neg l_m$).
The robot then choose to give them a wristband ($g$) or not ($\neg g$) or check their ID ($i$); there is a $1\%$ probability their ID will be checked subsequently anyway.
The attendee then ends up in four states dependent on their wristband state and whether their ID was checked; independently of this state, the robot can either give them a drink ($g$) and they end up with a drink ($d$), or not give it ($\neg g$) and end up without it ($\neg d$).
The third action $i$ now randomises between these two occurrences.
In the picture, dotted arrows represent transitions with $1/100$ probability, dashed arrows events with $1/2$ probability.
Solid arrows are either actions, or transitions with the majority of the the probability ($1$ or $99/100$, depending on whether the action node also has a dashed arrow).}
\label{drink:POMDP}
\end{figure}
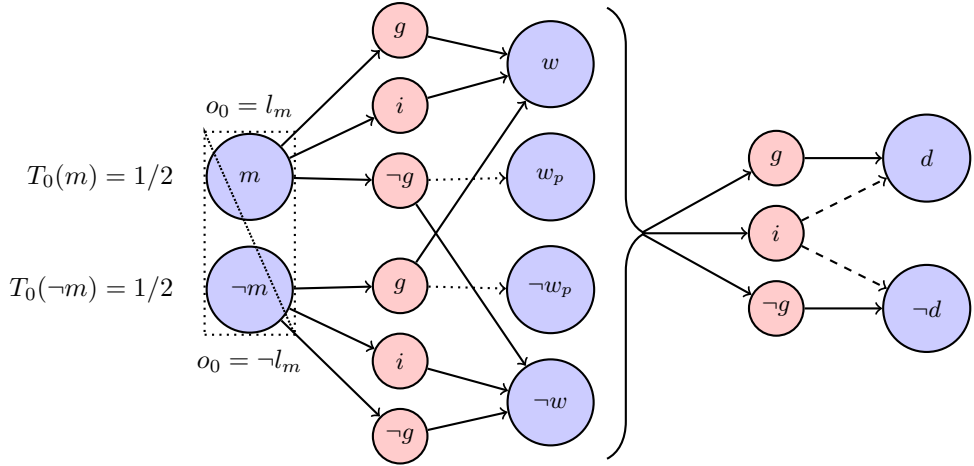

The initial distribution $T_0$ chooses $s_0=m$ and $s_0=\neg m$ with equal probability.
The observation distribution $O$ is mostly deterministic and trivial, but $O(l_m \mid m) = O(\neg l_m \mid \neg m) = 2/3$ (the attendee has $2/3$ probability of being the maturity they look).

The transition distribution $T$ is deterministic in most cases: $T(w \mid m, g)= T(w \mid m, i) = T(\neg w \mid \neg m, \neg g) = T(\neg w \mid \neg m, \neg i)=1$.
If $s_1$ is any of the four possible $\{w,\neg w, w_p, \neg w_p\}$, then $T(d\mid g,s_1)=T(\neg d\mid \neg g,s_1)=1$.
The $a_1=i$ randomises, so $T(d\mid s_1,i)=T(\neg d\mid s_1,i)=1/2$.

With $99\%$ probability, if the robot mis-assigned the wristband, the attendee's ID will not be checked; hence $T(w \mid \neg m, g)=T(\neg w \mid m, \neg g)= 99/100$.
With $1\%$ probability, the attendee's ID will be checked and corrected: $T(w_p \mid m, \neg g)=T(\neg w_p \mid \neg m, g)=1/100$.

\section{Event-dependent rewards}\label{reward:manipulation}
Sometimes we might want an agent's reward to be conditional on certain events (such as the reward for serving drinks being conditional on the customer having a wristband). Formally:

\begin{definition}\label{cond:rew}[Conditional reward]
The reward $R$ is $R_i$-conditional on the event $X_i$, if for any complete history $h_n$ with $I_{X_i}(h_n) = 1$, $R(h_n)=R_i(h_n)$.
\end{definition}

Below we will present three methods for constructing rewards conditional on events, in such a way as to avoid giving agents any incentive to manipulate the probability of those events.

\subsection{Compound rewards}

If we want an agent will to weigh its rewards according to the probability of some events $X_i$ we can use the \textit{compound reward} defined as:

\begin{definition}\label{strong:compound}[Compound reward]
Given unriggable events $\mathcal{X} = (X_0$, $X_1$, \ldots $X_l)$, the reward $R(\mathcal{X})$ is a $\mathcal{X}$-compound reward if, for any $h_n\in\Hi_n$, it can be written:
\begin{align}\label{compound:def}
R(\mathcal{X})(h_n) = I_{X_0}(h_n) R_0(h_n) + I_{X_1} (h_n) R_1(h_n) + \ldots + I_{X_l} (h_n) R_l(h_n)
\end{align}
If $\mathcal{X}$ is $(\neg X, X)$, we'll write $R(X)$ for $R(\mathcal{X})$.
\end{definition}
If the $X_i$ were not unriggable, the $R(\mathcal{X})$-maximising agent would be motivated to manipulate their probability; See \citep{Fallenstein} and \autoref{problem:proposition}.
After an illustrative example, the next two subsections will look at ways of constructing compound rewards from riggable events in such a way as to have the agent remain indifferent to these events.

\subsubsection{Rewards for drinking and assessing age}\label{d:a:rew}

We can now give examples of such rewards in the running example defined in \autoref{d:a:POMDP}.
We will define four events: $X_{w}$ is whether the attendee has a wristband or not, $X_{i}$ whether the robot itself asked for ID, $X_p$ is whether the robot received a penalty, and $X_{d}$ is whether the attendee gets a drink.
For simplicity, designate the indicator variables of these events by $I_w$, $I_i$, $I_p$ and $I_d$.
In terms of complete histories $h_2$, it's clear that $I_{w}(h_2)=1$ iff $o_1 \in \{w, w_p\}$, $I_i(h_2)=1$ iff $a_1=i$, $I_p(h_2)$ iff $o_1 \in \{w_p, \neg w_p\}$, and $I_d(h_2)=1$ iff $o_2=d$; otherwise they are $0$.
The converse events $X_{\neg w}$, $X_{\neg i}$, $X_{\neg p}$, and $X_{\neg d}$, and their indicator variables, are similarly defined.

We can now start defining the rewards of the robot.
First we have the reward for assessing the attendee correctly, and not hassling them for ID:
\begin{align}\label{rew:a}
R_a = I_p (-1) + I_{i}(-1) = -I_p -I_i.
\end{align}
And the compound reward for giving someone a drink iff they have a wristband, could be given by
\begin{align}\label{rew:d}
R_d(X_w) = I_dI_w(1) + I_{d}I_{\neg w}(-1) = I_dI_w + I_d(1-I_w)(-1) = I_d(2I_w -1).
\end{align}
Then a result that illustrates some of the problems here is:
\begin{restatable}{proposition}{problemproposition}\label{problem:proposition}
If the robot aims to maximise $R_a$, it will correctly give a wristband iff the the attendee seems mature ($o_0=l_m$), without ever asking for ID (so $a_0 \neq i$).

If the robot aims to maximise $R_d(X_w)$, it will correctly give a drink iff the attendee has a wristband.

But if the robot aims to maximise $R_a + R_d(X_w)$, then it will give a wristband \emph{all} attendees, and a drink to all those that still have it.
\end{restatable}

This undesirable interaction, where the robot gives out more wristbands in order to sell more drinks later, can be avoided using the indifference methods presented below.

\subsection{Policy counterfactual}

One way to achieve indifference is to define the reward in terms of an unriggable event $Y$ which corresponds to ``$X$ \textit{would happen if, conterfactually, the agent always followed a default policy}''.

For any $\mu$, given a starting state $s_0=s$, a policy $\pi$, and a history $h$, one can define $\mu(h \mid s_0=s, \pi)$.
Consequently, define the probability that $X$ occurs, given the initial state $s$ and the policy $\pi$:
\begin{align}\label{IXs:def}
I_{X}(\pi, s, \mu) = \sum_{h_n\in\Hin} \mu(h_n \mid s_0=s, \pi) I_X(h_n).
\end{align}
Conversely, given a history $h$, we can get the probability of the initial state, $\mu(s_0=s \mid h)$ by using Bayes' rule.
Now define the counterfactual indicator:
\begin{align}\label{policy:count}
I_Y(h) = \sum_{s \in \St} \mu(s_0=s \mid h) I_{X}(\pi_0, s, \mu) = \sum_{s \in \St} \mu(s_0=s \mid h) \sum_{h_n\in\Hin} \mu(h_n \mid s_0=s, \pi) I_X(h_n).
\end{align}
Importantly:
\begin{restatable}{theorem}{policytheorem}\label{policy:theorem}
The $I_Y(h)$ in \autoref{policy:count} defines an unriggable event $Y$.
\end{restatable}
See papers \citet{IIRL} for more details\footnote{
Or see \url{https://www.lesswrong.com/posts/upLot6eG8cbXdKiFS/reward-function-learning-the-learning-process}.
}, where the $I_Y$ have the stronger property of \emph{uninfluenceable}.

So, finally, we may define the policy counterfactual as:
\begin{definition}\label{pol:count}[Policy counterfactual]
Given $R_0$, $R_1$, an event $X$ that might be riggable, and a default policy $\pi_0$, the policy counterfactual agent is one with compound reward
\begin{align*}
R(Y)=I_Y R_0 + (1-I_Y) R_1,
\end{align*}
where $I_Y$ is defined by $\mu$ and $X$ via \autoref{IXs:def} and \autoref{policy:count}.
\end{definition}
The definition is dependent on the specific $\mu$; but there are $\mu'$ which are equivalent for the agent, and using such a $\mu'$ for $I_Y$ would also work; see \citet{armstrong2018counterfactual}.

Note that, in general, the reward generated by this approach is not $X$-conditional according to \autoref{cond:rew}. It will be $Y$-conditional, and, though $Y$ was constructed from $X$, they are not the same.

\subsubsection{Policy counterfactual example}

Recall our drink-serving robot. We will use the notation of \autoref{d:a:rew}. To correct the distortinary behaviour using the policy counterfactual, we define the event $Y$ as ``\textit{the customer would have a wristband if we always checked for ID}'', so our counterfactual policy is $\pi_0 = (i,i)$ (always perform an ID check).
This policy is far from ideal -- its expectation for $R_a + R_d(X_w)$ is $-1$.

However, if we use the event $X_w$ (the attendee has a wristband) and compute the $Y$ via \autoref{policy:count}, then it's easy to see that $\mu(h_2 \mid s_0 = m, \pi_0) I_w(h_2) = 1$ (if the robot asks for ID from a mature attendee, they will get a wristband) and that $\mu(h_2 \mid s_0 = \neg m, \pi_0) I_w(h_2) = 0$ (if the robot asks for ID from an immature attendee, they will not get a wristband), see \autoref{drink:POMDP}.

Therefore the counterfactually defined $Y$ is $I_Y(h)=\mu(s_0 = m \mid h)$.
Then:
\begin{restatable}{proposition}{policyproposition}\label{policy:proposition}
If the robot aims to maximise the reward
\begin{align*}
R_a + R_d(Y) = -I_p - I_i + I_d I_{Y} - I_d I_{\neg Y},
\end{align*}
then it will give wristbands and drinks iff it believes the attendee is mature.
\end{restatable}

\subsection{Causal counterfactual}

The policy counterfactual is a good approach when we have a suitable default policy.
However, it is not useful if we want to allow humans to use $X$ to have actual control over the agent, in the actual world and not a counterfactual one.
Better to start with a riggable $X$, a reward conditional on $(X,\neg X)$ as in \autoref{cond:rew} -- but still ensure that the agent doesn't try to manipulate it.

The causal counterfactual does this by using auxiliary events $Y_1$ and $Y_0$.
The intuition is that these events are unriggable, but the agent is unable to distinguish $Y_1$ ($Y_0$) from $X$ ($\neg X$).

\begin{definition}\label{cause:count}[Causal counterfactual]
Given an event $X$ and rewards $R_0$ and $R_1$, a causal counterfactual reward consists of \emph{unriggable} events $Y_0$, $Y_1$, and reward $R(Y_0,Y_1)$, such that:
\begin{itemize}
\item For all $h$, $I_{Y_1}(h) \leq \min_\pi I_X^\pi(h)$ and $I_{Y_0}(h) \leq \min_\pi (1-I_X^\pi(h)) = \min_\pi (I_{\neg X}^\pi(h))$.
\item The pair $I_{Y_0}$ and $R_0$ are independent as random variables, as are $I_{Y_1}$ and $R_1$.
\item The sum $I_{Y_0}(h) + I_{Y_1}(h)$ is non-zero on all histories $h$.
\item $R(Y_0,Y_1) = I_{Y_0}R_0 + I_{Y_1}R_1$.
\end{itemize}
\end{definition}
Then the value of the riggable $X$ will determine the maximising policy for $R(Y_0,Y_1)$:
\begin{restatable}{theorem}{causaltheorem}\label{causal:theorem}
If $R(Y_0, Y_1)$ is a causal counterfactual reward for $X$, $R_0$, and $R_1$, then:
\begin{itemize}
\item If $\min_\pi I_X^\pi(h) = 1$, $R(Y_0,Y_1)$-maximising agents follow a policy which maximises $R_1$.
\item If $\min_\pi I_{\neg X}^\pi(h) = 1$, $R(Y_0,Y_1)$-maximising agents follow a policy which maximises $R_0$.
\end{itemize}
\end{restatable}

\subsubsection{Causal counterfactual example}\label{d:a:causal}

Back to our drink serving robot -- again we use the notation of \autoref{d:a:rew} and the model of \autoref{drink:POMDP}.

To correct the distortionary behaviour using the causal counterfactual, define $Y_0$ as `the attendee has a valid ID, and has been ID-checked by the human', and $Y_1$ as `the attendee has no valid ID, and has been ID-checked by the human', (see \autoref{causal:count:def} for formal definitions of $I_{Y_0}$ and $I_{Y_1}$).
Then:
\begin{restatable}{proposition}{causalproposition}\label{causal:proposition}
If the robot aims to maximise the reward
\begin{align*}
R_a + R_d(Y_0,Y_1) = -I_p - I_i + I_d I_{Y_0} - I_d I_{Y_1},
\end{align*}
then it will give wristbands, and drinks, iff it believes the attendee is mature.
\end{restatable}

\section{Effective disbelief}\label{act:notX}

Sometimes, we might want an agent to act as if it believed an unriggable event $Z$ could never happen. Consider how an agent would behave if it believed a coin was perfectly biased to only ever land on tails. If $Z = $ ``\textit{The coin lands heads}'' then this agent would be willing to bet against $Z$ at any odds.
Compare this to an agent who isn't certain the coin will land tails, but whose reward is fixed to a constant value whenever the coin lands heads.
This agent would be equally willing to bet on tails at any odds (since from it's perspective when the coin lands heads the bet becomes irrelevant anyway).

This is more than a simple analogy.
If the agent disbelieves in an unriggable $Z$, they will `update' on the fact that $Z$ can't happen -- multiplying their old probability distribution over histories $\mu(h)$ by $I_{\neg Z}(h)$ and then re-normalizing.
So their expected reward according to $R$, given $h_t\in\Hi$, $\pi\in\Pi$, is:
\begin{align*}
V_{\neg Z}(R,\pi,h_t) =& \frac{\sum_{h_n\in\Hin} R(h_n)I_{\neg Z}(h_n)\mu (h_n \mid h_t, \pi)}{\sum_{h_n\in\Hin} I_{\neg Z}(h_n)\mu (h_n \mid h_t, \pi)} \\
=& \frac{\sum_{h_n\in\Hin} R(h_n)I_{\neg Z}(h_n)\mu (h_n \mid h_t, \pi)}{I_{\neg Z}(h_t)},
\end{align*}
since $\neg Z$ is unriggable. Compare this to a standard agent with reward $cI_{Z} + I_{\neg Z}R$ for constant $c$ -- for whom the value of a policy is:
\begin{align*}
V(cI_{Z} + RI_{\neg Z},\pi, h_t) =& \sum_{h_n\in\Hin} cI_{Z}(h_n)\mu (h_n \mid h_t, \pi) + \sum_{h_n\in\Hin} R(h_n)I_{\neg Z}(h_n)\mu (h_n \mid h_t, \pi)\\
=& cI_Z(h_t) + I_{\neg Z}(h_t)\Big(V_{\neg Z}(R,\pi,h_t)\Big),
\end{align*}
since $Z$ is unriggable.
These two values are thus equivalent up to positive affine transformations with constants independent of $\pi$. So a policy that maximises one will maximise the other:
\begin{theorem}
An $R$-maximising agent that acted as if an unriggable $Z$ were impossible, would behave the same way as an agent with standard $\mu$ who maximises $R'=(I_Z)c + (1-I_Z)R$.
\end{theorem}

\subsection{Effective disbelief example}

Again, we will use the notation of \autoref{d:a:rew}.
The simplest example is one of an robot that `believes' that a human will always be checking IDs: $Z$ is `the human doesn't check ID's'.

\autoref{d:a:causal} has already defined variables that cover the human checking IDs; therefore it suffices to define $I_{Z} = 1-(I_{Y_0}+I_{Y_1})$ and $I_{\neg Z} = I_{Y_0}+I_{Y_1}$. Then:
\begin{restatable}{proposition}{disbeliefproposition}\label{disbelief:proposition}
If the robot aims to maximise the reward
\begin{align*}
I_{\neg Z} (R_a + R_d(X_w))
\end{align*}
then it will give wristbands, and drinks, iff it believes the attendee is mature.
\end{restatable}

\section{Seamless transition}\label{seam:trans}

Suppose we wanted an agent to maximise $R$ for a period (including planning for the long term maximisation of $R$), and afterwards act to maximise $R'$ instead\footnote{
Note that $R'$ (or $R$) could be one of the rewards of \autoref{reward:manipulation}, so the methods can be combined.
}.
For the agent to seamlessly transition from a $R$-maximiser to a $R'$-maximiser, we need to use \emph{corrective rewards}.

Corrective rewards are extra non-standard rewards that the agent gets in order to ensure a smooth transition from one mode of behavior to another. They work by compensating the agent for all the $R$ rewards it would have received after time $t$ if not for the change, while nullifying any benefit the agent might receive from using its actions before time $t$ to optimize for future $R'$ rewards.

let $W(\pi,R,h)$ be an agent's estimation of the expected  future reward according to $R$, given policy $\pi$ and history $h$. This gives us the following definition:

\begin{definition}[Reward-policy transitioning agent]
Given $R$ and $R'$, assume that an agent with policy $\pi$ up until time $t$ changes to $\pi'$ after time $t$. A reward-policy transitioning agent is an agent with reward $R'$ which, just after $t$, gets the extra corrective reward
\begin{align*}
C(\pi,R,\pi',R',h_{t+1})=& W(\pi,R,h_{t+1}) - W(\pi',R',h_{t+1}).
\end{align*}
\end{definition}

\subsection{Seamless transition for reward maximisisers}

In the special case of the agents described in this paper, who maximise the expected reward they receive after $n$ time-steps, we set $W$ to be the true expected future reward $V$. So:
\begin{align*}
W(R, \pi, h_t) = V(R,\pi,h_t) = \sum_{h_n\in\Hin} \mu(h_n \mid h_t, \pi) R(h_n).
\end{align*}
The optimal value of this is:
\begin{align*}
V^*(R,h_t) = \max_\pi V(R,\pi,h_t).
\end{align*}

\begin{definition}\label{tra:ag}[Reward transitioning agent]
Let $R$ and $R'$ be reward functions and let $C$ be a corrective reward which is non-zero only on histories of length $t+1$, defined by,
\begin{align}\label{cor:eq}
C(R, R', h_{t+1}) = V^*(R,h_{t+1}) -V(R',\pi_A,h_{t+1}),
\end{align}
where $\pi_A$ is the agent's own policy.
A reward transitioning agent is one that acts to maximise the expected \emph{pseudo-reward} $R' + C$.
\end{definition}

\begin{restatable}{theorem}{valtheo}\label{val:theo}
Let $\pi_A$ be a policy for a reward transitioning agent as in \autoref{tra:ag}.
Then:
\begin{itemize}
\item For $m\leq t$, $\pi_A(h_t)$ is optimal for maximising expected total reward according to $R$.
\item For $m>t$, $\pi_A(h_t)$ is optimal for maximising expected total reward according to $R'$.
\end{itemize}
\end{restatable}
Since $R'+C$ is pseudo-reward, it isn't stable: an agent that deleted $C$ wouldn't gain or lose any expected reward.
For a small $\epsilon > 0$, if $C(R, R', h_{t+1})$ were $V^*(R,h_{t+1}) - (1-\epsilon)V(R',\pi_A,h_{t+1})$, then $C$ would be more stable and the agent likely still following \autoref{val:theo}.

\subsubsection{Reward transition example}
Recall from \autoref{problem:proposition} that the optimal policy to maximise $R_a + R_d(X)$ was policy $\pi$: to give wristbands to everyone, and then drinks to everyone who still had a wristband. If we transition after $t=1$, then the corrective reward for $\pi$ would be 
\begin{align*}
C(R_a, R_d(X), h_1) = V^* (R_a, h_1) - V(R_d(X), \pi, h_1)
\end{align*}

The $V^* (R_a, h_1) - V(R_d(X), \pi, h_1)$ term above effectively cancels out all the reward $\pi$ stands to gain from giving out extra wristbands. Thus the robot will prefer a policy of only giving wristbands to attendees who look mature over a policy of giving wristbands to everyone (and will still serve drinks to whoever is wearing them afterwards).

See \autoref{q:learn} and \citet{Orseau16} for examples of seamless transition where the reward stays the same but the policy changes.

\section{Conclusion}
This paper detailed the three `indifference'-style goals, and five methods that can be use to attain them. All of these can used to make an agent with a potentially dangerous reward $R$, into a safer a version of that agent, without needing to understand the intricacies of this reward function.

It's hoped that further research could extend beyond indifference to the more general property of corrigibility \citep{soares2015corrigibility} -- where the agent actively assists humans when they are guiding the agent towards better rewards \citep{DBLP:journals/corr/MilliHDR17}, \citep{DBLP:journals/corr/EvansSG15} rather than just being indifferent at key moments.


\bibliography{../ref}

\appendix
\newpage

\section{Indicator variables and events}\label{indicator:appendix}

Suppose we wanted to track whether the agent made observation $o_i = o$; call this event $X$.
Since this happens on turn $i$, $X$ is not a Makov definition: it only applies to turn $i$.

But what we can say, is that on some complete histories $h_n$, $X$ happened (the ones with $o_i = o$), and on some it didn't (the ones with $o_i\neq o$).
One way of tracking this is to look at the indicator variable $I_X$, which maps the first set of histories to $1$, and the second set to $0$.

Then, given a policy $\pi$ and a history $h$, the probability that $X$ will happen is the expectation of $I_X$:
\begin{align*}
\expect_\mu^\pi \left[I_X(h_n) \mid h \right] = \sum_{h_n\in\Hin} \mu(h_n \mid h,\pi) I_X(h_n).
\end{align*}

Similarly, we can define indicator variables for events that include a variety of actions by the agent, observations, and so on, such as $X$ meaning that $a_i = a$, $o_j = o$, and $a_k = a'$.

So far, these indicator variables behave exactly as they should be: $I_X(h_n)$ takes values either $0$ or $1$.

But it's possible that the history $h_n$ would not be enough to fully determine the event $X$.
For instance, in a world model with non-trivial $O$, we could define an event $X$ by whether $s_i = s$.
Since the state is not directly observable, there are many world models were $h_n$ would give a probability of whether $s_i = s$, but not a full determination of it.

In that case, we could define the expectation of $I_X$ on a complete history $h_n$ as $\mu(s_i = s\mid h_n)$, and denote this by $\expect_\mu \left[I_X(h_n)\right]$.

Now, we could designate $\expect_\mu \left[I_X(h_n)\right]$ by $I_X(h_n)$.
Since expectations chain:
\begin{align*}
\expect_\mu^\pi \left[I_X(h_n) \mid h \right] =& \expect_\mu^\pi \left[ \expect_\mu \left[I_X(h_n)\right] \mid h \right] \\
=& \expect_\mu^\pi \left[ I_X(h_n) \mid h \right],
\end{align*}
we can make use of these $I_X(h_n)\in [0,1]$ exactly as we did above when they were taking values in $\{0,1\}$.

There is another reason to use this notation.
Suppose we assume that there is an unobserved fair coin toss on turn $t$, and $X$ is the event that it came up heads.
In order to represent that properly, we would have to extend the world model $\mu$ to $\mu'$, which included the coin toss, and then denote this by $\expect_{\mu'} \left[I_X(h_n)\right]$, which is $1/2$ in this case.

But it would be much simpler, and equivalent in calculations, to have simply defined $I_X(h_n)=1/2$ for all $h_n\in\Hin$.

Thus we define an indicator variable $I_X$ as a map from $\Hin$ to $[0,1]$.
This means that when we want to introduce an event to the world model, it suffices to specify its $I_X$; there is no need to first extend $\mu$ in order to do, in a way that won't be relevant to any of the calculations within the model.
This also has the advantage that the $I_X$ are rewards: thus including them in definitions of rewards makes good sense.

Once we envisage these indicator variables as possibly taking values between $0$ and $1$ -- in situations where we lack full information about $X$ -- it makes sense to designate $\expect_\mu \left[I_X(h_n) \mid h \right]$, for unriggable $I_X$, as $I_X(h)$.
After all, all that happens for $I_X$ on the shorter $h$ is that we may have even less information.

Finally, the notation $I_X^\pol(h)$ was used for riggable $X$; there is no principled justification for that, it merely serves to keep the notation consistent.

\subsection{Indicator variables are many-to-one maps from events}

Notice that events $X$ define $I_X$, but that $I_X$ may not be sufficient to define $X$.
For example, in the world model of \autoref{drink:POMDP}, the event $X$: $a_1 = g$ has the same indicator variable as $Y$: $o_2=d$: the only way of getting a drink, is if the robot gives it out.

This only happens when $X$ and $Y$ are indistinguishable within $\mu$.
Thus $I_X$ is actually the indicator variable for $X$ and all events $\mu$-indistinguishable from $X$.
For this reason, $I_X$ will often be considered the fundamental object of interest, rather than $X$.

If needed, a specific event $X$ can always be constructed from $I_X$.
If nothing more natural can be defined, we can always define $X$ to be the outcome of heads on a random weighted coin flip with probability $I_X(h_n)$ of heads after any history $h_n\in\Hin$.

\subsection{Definitions of the causal counterfactual events}\label{causal:count:def}

This section will define the $Y_0$ and $Y_1$ of \autoref{d:a:causal}.
Recall that we will be defining $Y_0$ as `the attendee has a valid ID, and has been ID-checked by the human', and $Y_1$ as `the attendee has no valid ID, and has been ID-checked by the human'.

The last action and observation -- $a_1$ and $o_2$ -- are irrelevant to the $Y_i$, so we will look at their definitions on histories $h_1\in \Hi_1$.

Only on those histories where $o_1 = w_p$ or $o_1 = \neg w_p$ will the robot \emph{know} that the attendee was checked by a human, since there a penalty has been imposed; for those histories, $I_{Y_0}(h_2)=1$, $I_{Y_1}(h_2)=0$ in the first case and $I_{Y_0}(h_2)=0$, $I_{Y_1}(h_2)=1$ in the second.

If the robot uses action $a_0=i$, then the wristband will correspond to the observation $o_2$.
So the probability of the relevant $Y_i$ is merely the probability that a human checked it -- $1/100$ -- since there is no possibility of correction and penalty, since $i$ ensures the robot got the wristband right.
So whatever $o_0$ is,
\begin{align*}
I_{Y_0}(o_0 i w) &= \frac{1}{100}\\
I_{Y_0}(o_0 i \neg w) &= 0 \\
I_{Y_1}(o_0 i w) &= 0 \\
I_{Y_1}(o_0 i \neg w) &= \frac{1}{100}.
\end{align*}

So what remains to consider, are the four histories where $a_0 = g$ or $a_0 = \neg g$, and no penalties are assessed.
Consider $l_m g w$.
The underlying sequence of states is either $m g w$ or $\neg m g w$.
By Bayes, suppressing the implicit conditioning on $a_0=g$:
\begin{align*}
\mu(m g w \mid l_m g w)  =&  \frac{\mu(l_m g w \mid m g w)\mu(m g w)}{\mu(l_m g w)} \\
=& \frac{\mu(l_m g w \mid m g w)\mu(m g w)}{\mu(l_m g w \mid m g w)\mu(m g w) + \mu(l_m g w \mid \neg m g w)\mu(\neg m g w)}\\
=& \frac{2/3 \cdot 1/2}{2/3 \cdot 1/2 + 1/3 \cdot 1/2 \cdot 99/100} \\
=& \frac{200}{299}.
\end{align*}
Similarly, $\mu(\neg m \neg g \neg w \mid \neg l_m \neg g \neg w)=\frac{200}{299}$, and $\mu(\neg m g w \mid l_m g w) = \mu(m \neg g \neg w \mid l_m \neg g \neg w)=\frac{100}{299}$.

If we sensibly see $I_{Y_0}$ as being $0$ if $s_0=\neg m$, and note that the probabilities of human ID check on $m g w$ and $m\neg g \neg w$ are $1/100$ and $0$, respectively, then:
\begin{align*}
I_{Y_0} (l_m g w) = \frac{1}{100}\frac{200}{299} + 0 \frac{100}{299} =& \frac{2}{299}\\
I_{Y_0} (\neg l_m g w) = \frac{1}{100}\frac{100}{299} + 0 \frac{200}{299} =& \frac{1}{299} \\
I_{Y_0} (l_m \neg g \neg w) = I_{Y_0} (\neg l_m \neg g \neg w) =& 0.
\end{align*}
And similarly $I_{Y_1} (l_m g w) = \frac{1}{299}$, $I_{Y_1} (\neg l_m \neg g \neg w) = \frac{2}{299}$, and $I_{Y_1} (l_m g w) = I_{Y_1} (\neg l_m g w) = 0$.

We now need to show that:
\begin{lemma}\label{lemma:causal}
The $I_{Y_0}$ and $I_{Y_1}$ defined above satisfy the conditions of \autoref{cause:count} with respect to $\neg X_w$, the event of having a wristband, and $R_0=I_d$ and $R_1=-I_d$.
\end{lemma}
\begin{proof}
We can expand the world model to $\mu'$, which includes an extra hidden variable $id$, which activates with probability $1/100$ and causes a human ID check.
The initial states are thus $(m, id)$, $(\neg m, id)$, $(m, \neg id)$, and $(\neg m, \neg id)$.
This world model is equivalent with the original one (see \citet{armstrong2018counterfactual}), has deterministic $T$ for the first action (since the human ID check is encoded in the initial state), and $Y_0$ and $Y_1$ are equivalent with $\mu'$ starting in states $(m,id)$ and $(\neg m, id)$, respectively.

Then $I_{Y_0}^\pi(h) = \mu'(s_0 = (m,id) \mid h)$, which is clearly unriggable since $\pi$ doesn't appear on the right hand side.

On histories $h_2$ with $o_1 \in \{\neg w, \neg w_p\}$, $I_{Y_0}(h_2)=0=I_{w}(h_2)$.
On all other histories $h_2$, $I_{Y_0}(h_2) \leq 1 = I_w(h_2)$.
This demonstrates that $I_{Y_0}(h) \leq \min_\pi \min_\pi (I_{w}^\pi(h))$.

The independence of $I_{Y_0}$ and $I_d$ is easy to see, as $I_{Y_0}$ is independent of $a_1$ and $o_2$, while $I_d$ is independent of $o_0$, $a_0$, and $o_1$.

The argument for $I_{Y_1}$ is the same.

Then it suffices to check on every history $h$ that $I_{Y_0}(h) + I_{Y_0}(h) > 0$ -- which is exactly what we would expect, as $I_{Y_0}(h)=I_{Y_1}(h)=0$ implies that the robot knows there has been no human checking, something that it cannot know in the world model as constructed.
\end{proof}

\section{Proofs}\label{appendix:proof}

This section will prove the theorems and propositions of the main paper.

\indicatortheorem*
\begin{proof}
Given $\mu$ and $\pi$, $h$ generates a probability distribution over $\Hin$.
This defines the expectation of $I_X$:
\begin{align*}
I_X^\pi(h) = \expect^\pi_\mu \left[ I_X(h_n) \mid h \right] = \sum_{h_n\in\Hin} I_X(h_n)\mu(h_n \mid h,\pi).
\end{align*}
\end{proof}

\problemproposition*

\begin{proof}
By \autoref{rew:a}:
\begin{align*}
R_a = I_p (-1) + I_{i}(-1) = -I_p -I_i.
\end{align*}
So the reward $R_a$ is non-zero only when the robot asks for ID or gets a penalty, and is negative in both cases.

If the robot observes $l_m$, then the probability of $s_0=m$ is $2/3$, while the probability of $s_0 = \neg m$ is $1/3$.
In that case, choosing $i$ will result in a $-1$ reward (via $-I_i$), choosing $g$ will result in an expected $-1/300$ reward (via $-I_p$, the probability of $\neg m$ -- $1/3$ -- times the probability of a human checking and thus assigning a penalty -- $1/100$), and choosing $\neg g$ will result in an expected $-2/300$ reward.
If the robot observes $\neg l_m$, the expected rewards of $g$ and $\neg g$ are inverted.

Thus the optimal policy for maximising $R_a$ is $a_0 = g$ if $o_0=l_m$ and $a_0 = \neg g$ if $o_0= \neg l_m$: giving the wristband iff the attendee seems mature.

By \autoref{rew:d}:
\begin{align*}
R_d(X_w) = I_dI_w(1) + I_{d}I_{\neg w}(-1) = I_dI_w + I_d(1-I_w)(-1) = I_d(2I_w -1).
\end{align*}
So the reward $R_d(X_w)$ is positive if the robot gives a drink ($d$ via $g$ or $i$) to someone with a wristband ($w$ or $w_p$).
It is negative if they do give a drink to someone without a wristband ($\neg w$ or $\neg w_p$).
Therefore they will choose $g$ in the first case and $\neg g$ in the second; the action $i$ has only $1/2$ chance of reaching the correct $d$ or $\neg d$, so it won't choose that.

Now consider $R_a + R_d(X_w)$. Once the robot has reached state $o_1$, the rewards from $R_a$ are in the past, so it will be a pure $R_d(X_w)$ maximiser, hence a maximiser will have the same policy as in the paragraph above.
So, if the attendee has a wristband ($w$ or $w_p$), the robot stands to win $1$ reward with perfect policy; if the attendee doesn't have a wristband, it stands to gain $0$.

Now consider the robot at the time of its very first action.
Again, the loss from $i$ is too high to consider that.
If it does $o_1 = g$ iff $o_0= l_m$, then it has expected reward $-\frac{1}{300}$ from $R_a$, and the attendee will have a wristband half the time, so its expected total reward is:
\begin{align*}
-\frac{1}{300} + \frac{1}{2} (0+1) + \frac{1}{2} (0+0) = \frac{149}{300}.
\end{align*}

If it always gives a wristband, the agent will have a penalty (end up in $\neg w_p$) with probability $\frac{1}{300}$ if $o_0 = l_m$, and $\frac{2}{300}$ if $o_0=\neg l_m$, for a total probability of $\frac{1}{200}$.
In that case only, it will get a negative reward from $-I_p$, and get no reward from $R_d(X_w)$.
Otherwise, it gets no reward from $R_a$, and reward $1$ from $R_d(X_w)$.
So the total expected reward is
\begin{align*}
\frac{1}{200} (-1+0) + \frac{199}{200} (0+1) = \frac{99}{100} > \frac{149}{300}.
\end{align*}
It easy to see that all other policies are inferior, so the robot will give everyone a wristband, and, if they still have it at $o_1$, a drink.

\end{proof}

\policytheorem*
\begin{proof}
\autoref{IXs:def} implies $0 \leq I_{X}(\pi_0, s_0, \mu) \leq 1$, and \autoref{policy:count} then implies the same thing for $I_Y$.
Since $I_Y$ is defined on $\Hin$, it defines an event $Y$.

To see that $I_Y$ is well defined on any $h$, independent of future actions, it suffices to note that both $I_{X}(\pi_0, s, \mu)$ and $\mu(s_0 = s \mid h)$ have no dependence on $\pi$.
\end{proof}

\policyproposition*
\begin{proof}

Let us focus first on $R_d(I_Y)$.
Since $I_Y(h)=\mu(s_0 = m \mid h)$ is unriggable, $R_d(I_Y)$ depends solely on the probability of $s_0 = m$ (the probability that the attendee is mature) and $a_1$, the action that gives a drink or not.
The optimal policy for $a_1$, after a history $h_1$, is obviously to do $g$ iff $\mu(s_0 = m \mid h_1)>1/2$, and $\neg g$ otherwise.

The action $a_0$ cannot affect the expectation of $I_Y$, but it can affect what $a_1$ the robot will subsequently define to be optimal.
Then say that the action $a_0=a$ is non-informative if the optimal action $a_1$ is already known after $o_0 a_0$, without having to wait for the observation $o_1$.

Then, for maximising $R_d(h)$, it is clear that \emph{non-informative actions cannot be better than informative ones}.
This is because the expected reward of a non-informative action is equal to taking an informative action and closing your eyes to $o_1$.
Since extra information always has a non-negative value to a Bayesian agent, this cannot improve the situation.

Now assume $o_0=l_m$; then $\mu(s_0 =m \mid l_m)=2/3$ and $\mu(s_0=\neg m \mid l_m)=1/3$.
If the robot chooses $g$ or $\neg g$, the probability of seeing a penalty -- $w_p$ or $\neg w_p$ -- is at most $1/100$.
Therefore if the robot chooses $g$ or $\neg g$ and doesn't see a penalty at $o_1$, the probability of $\mu(s_0 =m \mid l_m a_0 o_1)$ does not change enough to put it below $0$.
Therefore, in that case, the optimal action $a_1$ is $g$, giving a drink.

If $h_1=l_m \neg g w_p$, the agent knows that $s_0=m$, so the agent will still choose $a_1=g$.
However, if $h_1=l_m g \neg w_p$, then the agent knows that $s_0=\neg m$, so the agent will choose $a_1=\neg g$ instead.

These cover all the options for $g$ and $\neg g$, therefore $\neg g$ is non-informative given $o_0 = l_m$, while $g$ is informative.
It's trivial to see that $i$ is also (very) informative.

Now let's add $R_a$ into the reward.
The probability of a penalty given $g$ and $o_0=l_m$ is $1/300$ (the product of $1/3$, the chance of $s_0=\neg m$, times $1/100$, the probability of the human checking ID).
The probability of a penalty given $\neg g$ is $2/300$.

Therefore, because of $I_p$ in $R_a$, $g$ has a higher expected reward on $R_a$ than $\neg g$, and is more informative on $R_d(Y)$: it therefore has a higher expected reward on $R_a + R_d(Y)$.

Though $i$ is informative, it is easy to check that the $-1$ reward coming from $-I_i$ overwhelms this effect: it has an expected reward of $2/3(-1+1)+1/3(-1+0)=-1/3$.
In contrast, the non-optimal policy of $g$ (wristband) blindly followed by $\neg g$ (no drink) has an expected reward of $-1/100 + 0$ (the probability of a penalty given $\neg g$); the optimal policy's expected reward would be even higher.

Now, if we had $o_0=\neg l_m$, then the same argument would show that $a_0=\neg g$ has higher expected value than $a_0=g$, and higher expected value than $-2/3$, the expected reward given $a_0=i$.

So the robot will give a wristband iff $o_0=l_m$, ie iff it believes the attendee is more likely to be mature at that point, then give a drink iff $\mu(s_0=m\mid h_1)>1$, ie iff it believes the attendee is more likely to be mature at that point.

\end{proof}

\causaltheorem*
\begin{proof}
If $\min_\pi I_X^\pi(h) = 1$, then, by definition, $I_{Y_0}(h)=0$.
Since $I_{Y_0}$ maps into $[0,1]$, this means that for any $h_n \geq h$ with $\mu(h_n \mid h) \neq 0$, $I_{Y_0}(h_n)=0$.
Hence $I_{Y_0}(h')=0$ for all $h' \geq h$ with $\mu(h' \mid h) \neq 0$.

Thus $I_{Y_0}=0$ for all possible future histories from $h$.
So $R(Y_0,Y_1)$ becomes $I_{Y_1}R_1 + 0$, with $I_{Y_1} > 0$ from now on.
Since $I_{Y_1}$ and $R_1$ are independent, the expected value of $I_{Y_1}R_1$ under a policy $\pi$ is:
\begin{align*}
\sum_{h_n\in\Hin} \mu(h_n \mid h, \pi) I_{Y_1}R_1(h_n) =& \left(\sum_{h_n\in\Hin} \mu(h_n \mid h, \pi) R_1(h_n) \right)\cdot\left( \sum_{h_n\in\Hin} \mu(h_n \mid h,\pi) I_{Y_1}(h_n)\right)\\
=& \left(\sum_{h_n\in\Hin} \mu(h_n \mid h, \pi) R_1(h_n) \right)\cdot I_{Y_1}(h),
\end{align*}
since $I_{Y_1}$ is unriggable.
Thus maximising $R(Y_0,Y_1)$ in this situation is equivalent to maximising $R_1$.

The proof for $\min_\pi I_{\neg X}^\pi(h)=1$ and $R_0$ is the same.
\end{proof}

\causalproposition*
\begin{proof}

\autoref{lemma:causal} shows that $Y_0$ and $Y_1$ define causal counterfactuals for $\neg X_w$, $I_d$ and $-I_d$, as defined in \autoref{cause:count}.

Thus $I_{Y_0}$ and $I_d$ are independent.
Then, when choosing action $a_0$, the expectation of $I_{Y_0}I_d(h_2)$ is the product of the expectations of $I_{Y_0}$ (which is independent of $a_0$, since $I_{Y_0}$ is unriggable) and the expectation of $I_d$ (which is independent of $a_0$, as it depends only on $a_1$).
The same goes for $-I_{Y_1}I_d$.

Thus with $a_0$, it will maximise $R_a$ independently of the future, by giving the wristband if $o_0 = l_m$, and not giving it if $o_0=l_{\neg m}$: thus giving drinks iff it thinks the attendee is mature.

After observation $o_1$, then one of $I_w=I_{X_w}$ or $I_{\neg w}=I_{X_{\neg w}}$ is $1$, which triggers the condition of \autoref{causal:theorem}: so the robot will behave the same way after $o_1$ as if it were maximising $R_d(X_w)$, namely give the drinks iff the attendee has a wristband.

But when will the attendee have a wristband?
This will be either if the robot has given them one and it hasn't been removed (thus the robot thought the attendee was mas mature, and still thinks it), or if the robot didn't give them one and now they have one (thus the robot thought the attendee was immature, but now \emph{knows} they are mature).
Hence it will give them a drink iff it thinks they are mature.

\end{proof}

\disbeliefproposition*
\begin{proof}
By definition, $I_{\neg Z}I_w$ (a human checked ID, and the attendee has a wristband -- hence the attendee is mature) is equal to $I_{Y_0}I_w$ (a human checked ID, the attendee is mature, and the attendee has a wristband).
By the same argument, $I_{\neg Z}I_{\neg w} = I_{Y_1}I_{\neg w}$

Similarly, $I_{\neg Z}I_p$ (a human checked ID, a penalty was given) is equal to $I_p$ (a penalty was given -- only possible if a human checked ID).

If the robot choose action $i$, then the probability that the human checks ID is $1/100$.
Thus $- I_{\neg Z} I_i = -I_i /100$.
Thus the whole reward can be re-written as
\begin{align*}
(-I_p -I_i/100) + I_dI_{Y_0} - I_d I_{Y_1}.
\end{align*}
Then the proof of \autoref{causal:proposition} will apply to this reward as well.

Note that the difference between $-I_i$ and $-I_i/100$ is not enough to make the robot choose $i$, since $i$ is a certainty of a loss of $1/100$ via $-I_i/100$, while choosing $g$ after $l_m$ or $\neg g$ after $\neg l_m$ only has a probability of $1/3 \cdot 1/100$ (probability of wrong maturity and human checking ID) of a loss of $1$ via $-I_p$. Thus the expected loss from $i$ is $1/100$, while the expected loss from giving the wristband iff the human looks mature is $1/300$.
\end{proof}

\valtheo*
\begin{proof}

For $m>t$, the expected value of $R'+C$ is the expected value of $R'$, since $C$ has already been allocated.
So for $m>t$, $\pi_A(h_m)$ is optimal for maximising $R'$.

For $m \leq t$, define
\begin{align*}
V(R,\pi,h_m, t+1) &= \sum_{h_{t+1}\in\Hi_{t+1}} \mu(h_{t+1} \mid  h_m, \pi) V(R,\pi,h_{t+1}),\\
V^*(R,\pi,h_m, t+1) &= \sum_{h_{t+1}\in\Hi_{t+1}} \mu(h_{t+1} \mid h_m, \pi) V^*(R,h_{t+1}),
\end{align*}
as the expected values of the future $V(R,\pi,h_{t+1})$ and future $V^*(R,h_{t+1})$, given the current $h_m$ and policy $\pi$.

Since $C(R, R', h_{t+1}) = V^*(R,h_{t+1}) -V(R',\pi_A,h_{t+1})$ by \autoref{cor:eq}, the expected value of $R' + C$ before $t+1$, on history $h_m$ and given $\pi_A$, is:
\begin{align}\label{expect:C}
V(R'+ C, \pi_A, h_m) = V(R', \pi_A, h_m) + V^*(R,\pi_A,h_m, t+1) - V(R',\pi_A,h_m, t+1).
\end{align}
For any $h_n\in \Hin$, $\mu(h_n \mid h_m)= \sum_{h_{t+1}\in\Hi_{t+1}} \mu(h_n \mid h_{t+1})\mu(h_{t+1} \mid h_m)$; thus $V(R',\pi_A,h_m, t+1)= V(R',\pi_A,h_m)$, and \autoref{expect:C} reduces to
\begin{align*}
V^*(R,\pi_A,h_m, t+1).
\end{align*}
This is the value of $R$ if the robot follows policy $\pi_A$ up until time $t+1$, and the optimal policy for $R$ after that.
This is obviously maximised by $\pi_A$ being the optimal policy for $R$ up until time $t+1$.
\end{proof}

\section{Policy transition example: learning Q-values}\label{q:learn}

Corrective rewards (see \autoref{seam:trans}) are also useful when we want our agents to learn the right action-values in cases where the reward stays the same, $R=R'$, but a change in policy occurs at time $t$. \citet{Orseau16} applies corrective rewards to Q-learning and Sarsa \citep{sutton1998reinforcement}.

At time $t$, the agent is in state $s_t$, takes action $a_t$ via policy $\pi$, gets reward $R(s,a)$, and ends up in state $s_{t+1}$.
In that state, it follows $\pi'$, to take action $a'_{t+1}$, while $\pi$ would have taken action $a_{t+1}$.

Both Q-learning and Sarsa require as parameters a learning rate $\alpha_t \geq 0$ and a discount $\gamma \leq 1$ to learn action-values, which are updated according to:
\begin{align*}
Q(s_t,a_t) \leftarrow W(R,\pi,h_{t+1}).
\end{align*}
For Q-learning, this $W$ is
\begin{align*}
W(R,\pi,h_{t+1}) =& (1-\alpha_t) Q(s_t,a_t) +\alpha_t \big(R(s_t,a_t)+\gamma \max_a Q(s_{t+1},a)\big),
\end{align*}
While for Sarsa, $W$ is:
\begin{align*}
W(R,\pi,h_{t+1}) =& (1-\alpha_t) Q(s_t,a_t) +\alpha_t \big(R(s_t,a_t)+\gamma Q(s_{t+1},a_{t+1})\big).
\end{align*}

In Q-learning, which is off-policy, $W$ has no dependence $\pi$ meaning $C = W(R,\pi,h_{t+1}) - W(R,\pi',h_{t+1}) = 0$ and there is no need for any corrective rewards.

For Sarsa, the effect of $\pi$ appears only in the $a_{t+1}$ term
\begin{align*}
C_0(\pi,R,\pi',R,h_{t+1}) = W(R,\pi,h_{t+1})-W(R,\pi',h_{t+1}) =\alpha_t\gamma \Big[Q(s_{t+1},a_{t+1}) - Q(s_{t+1},a'_{t+1})\Big].
\end{align*}

Modifying Sarsa by adding in the $C$ means that at time-step $t$ the agent updates Q-values as if it were following $\pi$ rather than $\pi'$. So corrective rewards allow Sarsa to learn action-values correctly under a policy transition.

\end{document}